\documentclass[letterpaper]{article} 
\usepackage{aaai24}  
\usepackage{times}  
\usepackage{helvet}  
\usepackage{courier}  
\usepackage[hyphens]{url}  
\usepackage{graphicx} 
\usepackage{natbib}  
\usepackage{caption} 
\usepackage{algorithm}
\usepackage{algorithmic}

\usepackage{newfloat}
\usepackage{listings}
\DeclareCaptionStyle{ruled}{labelfont=normalfont,labelsep=colon,strut=off} 
\floatstyle{ruled}
\newfloat{listing}{tb}{lst}{}
\floatname{listing}{Listing}
\usepackage{xcolor}
\usepackage{amsmath}
\usepackage{amsfonts}
\usepackage{amsthm}
\usepackage{multirow}
\usepackage{subcaption}
\usepackage{lipsum}

 \nocopyright 

\makeatletter
\def\blfootnote{\xdef\@thefnmark{}\@footnotetext}
\makeatother

\title{Epsilon*: Privacy Metric for Machine Learning Models}
\author {
   Diana M.  Negoescu\textsuperscript{\rm 1},
   Humberto Gonzalez\textsuperscript{\rm 1},
   Saad Eddin Al Orjany\textsuperscript{\rm 1},
   Jilei Yang\textsuperscript{\rm 1},  \\
   Yuliia Lut\textsuperscript{\rm 1},
   Rahul Tandra\textsuperscript{\rm 1},
   Xiaowen Zhang\textsuperscript{\rm 1},
   Xinyi Zheng\textsuperscript{\rm 1},
   Zach Douglas\textsuperscript{\rm 1},
   Vidita Nolkha \textsuperscript{\rm 1},
   Parvez Ahammad\textsuperscript{\rm 1}, 
   Gennady Samorodnitsky\textsuperscript{\rm 1, \rm 2}  
}
\affiliations {
    \textsuperscript{\rm 1}LinkedIn Corporation\\
    \textsuperscript{\rm 2}Cornell University\\
}

\newtheorem{thm}{Theorem}
\newtheorem{definition}{Definition}

\newcommand{\R}{\mathbb{R}}
\newcommand{\calZ}{\mathcal{Z}}

\begin{document}

\maketitle

\begin{abstract}
We introduce Epsilon*, a new privacy metric for measuring the privacy
risk of a single model instance prior to,  during, or after deployment of privacy mitigation strategies.  {\color{black}The metric requires only black-box access to model predictions,  does not require training data re-sampling or model re-training, and can be used to measure the privacy risk of models not trained with differential privacy}.  {\color{black} Epsilon* is a function of true
  positive and false positive rates in a hypothesis test used by an
  adversary in a membership inference attack.  We distinguish between quantifying
  the privacy loss of a trained model instance, {\color{black} which we refer to as \textit{empirical privacy},}  and quantifying the privacy loss of the
  training mechanism which produces this model instance.  Existing approaches in the privacy auditing literature provide lower bounds for the latter,  while our metric provides an empirical lower bound for the former by relying on an  ($\epsilon,
  \delta$)-type of quantification of the privacy of the trained
  model instance.  We establish a relationship between these
  lower bounds} and show how to implement Epsilon* to avoid numerical and noise amplification instability.  We further show in experiments on benchmark public data sets that Epsilon* is sensitive to privacy risk mitigation by training with {\color{black} differential privacy} (DP), where the value of Epsilon* is reduced by up to 800\% compared to {\color{black}the Epsilon* values of} non-DP trained baseline models.   This metric allows {\color{black} privacy auditors to be independent of model owners,  and enables} visualizing the privacy-utility landscape to make informed decisions regarding the trade-offs between model privacy and utility.\blfootnote{Accepted at \textit{PPAI-24: The 5th AAAI Workshop on Privacy-Preserving Artificial Intelligence},  Vancouver,  Canada,  2024.}
\end{abstract}

\section{Introduction}
\label{sec:intro}
Many machine learning (ML) models are trained using either private customer data (age, gender, but also  quasi-identifiers such as geographical location), or using third-party information.  In order to offer data-driven privacy protections {\color{black} in the context of} increasing demand for privacy protection,  model owners need metrics that help quantify the privacy risk in their {\color{black} artificial intelligence} models,  in addition to privacy risk mitigation strategies which reduce the risk once the metrics identify it as high. 

In this paper,  {\color{black} we start with an  ($\epsilon,
  \delta$)-type of quantification of the privacy of a trained model instance, {\color{black}which we call \textit{empirical privacy}, }and derive a lower bound for the 
  $\epsilon$ in this quantification} 
 inspired by the hypothesis test characterization of differential privacy (DP) in \cite{kairouz2015composition}. We use this lower bound,  which we refer to as Epsilon*,  as a privacy metric in and of itself,  as it has several desirable properties.  Specifically,  Epsilon*:
\begin{itemize}
\item Does not interfere with the training pipeline,  i.e.,  it does not require model re-training or re-sampling of the training set. While obtaining statistically valid lower bounds on the DP parameter $\epsilon$ requires model re-training,  or,  at a minimum,  re-sampling the training set to introduce canaries \cite{stenke2023onerun}, obtaining a valid lower bound on model privacy does not require either.  This is important, as privacy auditors might not have access to {\color{black}altering} the input or training pipeline of models,  and deployed  models in the industry are often very large and require many hours {\color{black} or days} to train,  rendering their re-training prohibitive.
\item {\color{black}Enables model owners to prioritize which training pipelines to be mitigated with DP.  Epsilon* can be computed on any ML model instance,  including on model instances from training mechanisms where DP hasn't been deployed yet.} DP is the gold standard privacy mitigation strategy,   but comes with well-documented computational and utility drawbacks, {\color{black}e.g. }\cite{denison2022private, ponomareva2023dp}.  {\color{black} We make a distinction between \textit{mitigating} and \textit{measuring} privacy risk of models: in particular, training with DP is a mitigation strategy, while Epsilon* provides a measurement of privacy loss on a given model instance.  Even though training with DP ensures model-level privacy and its $\epsilon$ parameter in DP is a measurement of worst-case privacy loss,  }it is not trivial for auditors agnostic of the training mechanism to determine whether or not DP was used in training and if so,  with what $\epsilon$ parameter,  as recent advances in the area of auditing DP \cite{stenke2023onerun, nasr2023tight, zanella2022bayesian} assume knowledge of whether or not DP was used,  in addition to access to the training pipeline for re-sampling or model re-training.  
\item Can quantify privacy at all stages of model development and deployment, including when privacy mitigation strategies do not involve DP.  For instance,  model owners may investigate the effects on privacy when iterating through model development stages of feature addition or removal, or of training the model for more or fewer epochs.
\end{itemize}

In this paper,  we introduce Epsilon* and show how to derive it from the true positive (TPR) and false positive rates (FPR) of a hypothesis test of a simulated membership inference attack (MIA) on population data (MIA-PD) \cite{shokri2017membership, yeom2018privacy}, where an attacker attempts to predict whether a given data point was used for model training based on the training and non-training (population) loss distributions from a given model instance.  We note that the hypothesis test formulation is not restrictive to this particular attack,  and in fact Epsilon* could be derived from any other attacks that can be formulated as hypothesis tests on neighboring data sets,  such as attribute inference attacks (AIAs) \cite{yeom2018privacy}.  We chose MIA-PD as it is a canonical privacy attack,  requiring minimum (black box) model access by the attacker,  and is the foundation of other attacks such as AIAs \cite{Carlini2022membership}.

We map challenges in the implementation of Epsilon*, coming from the fact that often the most impactful TPR and FPR realizations are those close to 0 and 1,  where sampling noise and numerical floating point errors are prone to occur. We show that we can avoid these pitfalls by working with parametric distributions fitted to a transformation of the loss data, where an analytical formulation to Epsilon* is also available.

{\color{black}Our main contribution is on the empirical side,  where we compute Epsilon* for more than 500 model instances trained on public data sets (UCI Adult, Purchase-100),  with varying privacy mitigation strategies (without DP  and with DP for various DP $\epsilon$ values), and a range of model hyperparameter configurations. Our experiments show that Epsilon* is sensitive to privacy mitigation with DP,  and its value was below that of the $\epsilon$ used when training with DP in all public data set experiments.  On the theoretical side,  we prove that a statistical relationship between Epsilon* and a lower bound on the privacy loss parameter $\epsilon$ of the training mechanism exists. }

\section{Background and Related Work}
\label{sec:related_work}

{\color{black} The standard definition of differential privacy of a randomized
mechanism $M$ taking as an argument a database $D$ and outputting a
value in some set $\calZ$ \cite{dwork2006calibrating}, is
\begin{equation} \label{e:DP.dwork}
  Pr[M(D_0)\in R]\leq e^\epsilon Pr[M(D_1)\in R] +\delta
\end{equation}
for all pairs of neighboring databases $D_0$ and $D_1$ that
differ in the record of a single individual and 
for every (measurable) subset $R\subset\calZ$. A mechanism with this
property is called $(\epsilon,\delta)$-DP. 

A  powerful hypothesis test characterization of DP was introduced by
\cite{kairouz2015composition},  {\color{black} and later expanded by \cite{dong2022gaussian} in {\color{black} defining} Gaussian differential privacy}. Here, one views a set $R$ as the
rejection (critical) region of the statistical test
\begin{align*}
H_0:&\ \text{the mechanism was run on the database } D_0,\\
H_1:&\ \text{the mechanism was run on the database } D_1
\end{align*}
for some fixed neighboring databases $D_0$ and $D_1$.}
\begin{thm}[\cite{kairouz2015composition}]\label{Kairouzhm}
A mechanism $M$ is $(\epsilon,\delta)$ - DP if and only if the following conditions are satisfied for all pairs of neighboring databases $D_0$ and $D_1$, and all rejection regions $R$:
\begin{align*}
Pr[M(D_0) \in R] + e^{\epsilon} Pr[M(D_1) \in \bar{R}]  &\geq 1- \delta \\
Pr[M(D_1) \in \bar{R}] + e^{\epsilon} Pr[M(D_0) \in R]  &\geq 1-   \delta,
\end{align*}
\end{thm}{\color{black}where $\bar{R}$,  the acceptance region, is the complement of $R$.} 
Here,  $Pr[M(D_0) \in R]$ is the false negative rate (FNR, {\color{black} Type I error)}
and $Pr[M(D_1) \in \bar{R}]$ is the false positive rate (FPR, {\color{black} Type II error)} {\color{black}for the hypothesis test above.} 

MIAs naturally define a hypothesis test for whether a target entry belongs to
the training set or not \cite{shokri2017membership,  yeom2018privacy,
  Carlini2022membership}. However, it is important to realize that
the test is based on the loss of the trained model on a query point,
not the change in the model weights resulting from adding or removing
the query point from the training data set. 
 In MIA-PD,  attackers usually choose a threshold $\tau$,
and classify any point whose model loss  is above $\tau$ as
non-training: the test statistic is then the model loss and the
corresponding rejection region $R$ for this hypothesis test is the interval $[\tau,  \infty)$ \cite{yeom2018privacy,  salem2018ml, ye2022enhanced}. 

Most advances in the growing area of auditing differential privacy
have also been built upon hypothesis testing framework inspired by 
Theorem \ref{Kairouzhm}: \cite{ding2018detecting,  bichsel2018dp,
  jagielski2020auditing,  Nasr2021lowerbounds,  tramer2022debugging,
  zanella2022bayesian}. Indeed, estimating a pair of (FNR, FPR) in the
hypothesis test of  Theorem \ref{Kairouzhm} provides a lower bound on
the $\epsilon$ parameter of the mechanism $M$, which can be used  to
detect mistakes in the implementation of the mechanisms leading to
higher privacy loss values than intended.  However, if $M$ is a ML
mechanism, then 
multiple runs of re-sampling and re-training the models are neccessary for the statistical esimation of the FNRs and FPRs,  ranging from the order of thousands when Clopper-Pearson methods are used in statistical estimation \cite{Nasr2021lowerbounds,  tramer2022debugging} to about 500 when Bayesian methods are employed \cite{zanella2022bayesian}.

Several suggestions on evaluating the privacy of training mechanisms while reducing the number of times models need to be re-trained in order to estimate the lower bound on the $\epsilon$  have been made recently \cite{lu2022general,  stenke2023onerun}, in particular in the area of federated learning \cite{andrew2023one,  maddock2022canife,  nasr2023tight}.  Several of these methods now only require one  \cite{stenke2023onerun}
 or two models \cite{ nasr2023tight} to be re-trained, but all of
 them still require access to the sampling of the data sets used for
 model training for introducing canaries \cite{stenke2023onerun,
   maddock2022canife},  and/or visibility into model updates during
 training \cite{nasr2023tight}. As we pointed out above, measuring
 privacy in the ML context is difficult,  and not all methods in the literature address the same type of resilience against an attack.

In {\color{black} this paper}, we propose an estimator for the lower bound on the {\color{black} privacy risk} against
MIAs of a single model instance rather than on the privacy {\color{black} risk} of the
{\color{black}training mechanism that generates the model instances}.  This enables privacy auditors to be independent of model owners when evaluating the privacy loss of a model,  paving the way for transparent and easy to implement measurements of privacy {\color{black} risk}. We advocate for using this estimator as a privacy metric in and of itself,  as it shows sensitivity to privacy mitigation strategies such as DP when evaluated on deep neural net model instances trained on public data sets.

%
%
%
\section{Methods}
\label{sec:methods}

{\color{black} We adapt the definition of differential privacy in
  \eqref{e:DP.dwork} for ML applications where we are
  interested in measuring the resilience of a trained ML model instance
  against an adversary performing a membership inference attack on population data (MIA-PD). } Let ${\color{black}\mathcal{X} \times \mathcal{Y}}$ be the set of all possible observations $(x, y)$ a ML training mechanism $M$ may use for training,  testing and/or validation, where $x$ is usually a feature vector {\color{black} in $\mathbb{R}^d$} and $y$ is the label.  {\color{black} Let $\mathcal{P}$ be a distribution over  $\mathcal{X} \times \mathcal{Y}$.} The training mechanism takes as input a training data set ${\color{black}D =} \{ (x_1, y_1), ..., (x_n, y_n)\} \in ({\color{black}\mathcal{X} \times \mathcal{Y}})^n$ {\color{black} sampled according to  ${\color{black}\mathcal{P}}$} and outputs a trained model with weights $\theta = M(D) \in \Theta$, where $\Theta$ is the space of all possible weights,  {\color{black}i.e.,  $\theta$ is an observation of a (usually) random training mechanism $M$ applied to data set $D$}.  The prediction of the model at observation $x$ is then $f_{\theta}(x)${\color{black}: for example, in binary classification models,  $f_{\theta}(x) \in {\color{black}\mathcal{Y}} = [0,1]$ is typically the probability of observing label $y$ = 1,  where we can think of $y$ as taking values 0 or 1 in ${\color{black}\mathcal{Y}}$}. 

The MIA adversary formulates the hypothesis test:
\begin{align*}
H_0:& \text{ The training data set does not contain point }(x,y),\\
H_1:& \text{ }(x,y) \text{ belongs to the training data set.}
\end{align*}
{\color{black} The adversary has black-box access to the trained model
  and makes a decision based on evaluating $f_\theta(x)$ and comparing
  it with $y$.  {\color{black}Therefore, the adversary utilizes a rejection region
  that depends on $\theta$ only through the model output $f_{\theta}(x)$}. The adversary chooses a subset $R\subset {\color{black}\mathcal{Y}}\times {\color{black}\mathcal{Y}}$ and rejects $H_0$ (i.e. decides that
  the query point $(x,y)$ was, in fact, used for training) if
  $(f_\theta(x),y)\in R$. 
  
{\color{black} Consistent with the MIA literature  \cite{jayaraman2020revisiting},  we assume that the MIA attacker chooses $R$ by first considering a loss function {$l(f_{\theta}(x), y)$ }of the model with weights $\theta$ under input $(x,y)$.\footnote{We use $l(f_{\theta}(x), y) =
  (1-2y)(\log(f_{\theta}(x)) - \log(1-f_{\theta}(x)))$ for binary
  classification models and \\$l(f_{\theta}(x), y) =
  -\sum_j(\log(f_{\theta}(x)_j) -
  \log(1-f_{\theta}(x)_j))\mathbf{1}_{\{y_j = 1\}}$ for multi-class
  classification models, where $j$ is an index over the classes.}
The output $\theta$ of the model
generates a \textit{population}  loss distribution resulting from
sampling a point { from ${\mathcal{P}}$} and evaluating the loss of the model on that point. 
 Similarly,  the output $\theta$ of the model generates a
 \textit{training data} loss distribution that results from sampling at random
 a point from the training data set $D$ and evaluating the loss of the
 model on that point.  The losses on the training data set
 tend to be smaller than the losses on non-training (population) data
 \cite{jayaraman2020revisiting},  therefore { the
   adversary selects a threshold $\tau$ and reject $H_0$
   (i.e. conclude that the query point $(x,y)$ was used for training)
   if the loss $l(f_{\theta}(x), y)\leq \tau$. This corresponds
   to setting}
 {\color{black}
      \begin{align} \label{R_threshold}
R=R(\theta) &= \{(y_1, y_2) \in {\color{black}\mathcal{Y}} \times {\color{black}\mathcal{Y}}: l(y_1, y_2)
                  \leq \tau \}.
      \end{align}     }
}
Our quantification of resilience of the trained model is given
below. It is inspired by the  $(\epsilon,\delta)$-{\color{black}parameterization}  of
differential privacy in \eqref{e:DP.dwork}, and allows us to derive Epsilon* as a lower bound for the $\epsilon$ in this quantification.}
{\color{black}
\begin{definition}[Empirical privacy]\label{model_privacy_def}
A model { instance parameterized by $\theta$,  and obtained from using a training mechanism on data set $D = \{(x_i, y_i), i = 1, ..., n\}  {  \in  (\mathcal{X} \times \mathcal{Y}})^n$},  is $(\epsilon,  \delta)$-empirically private if for any set $R $ of the form \eqref{R_threshold}
{
\begin{align} 
&\mathbb{P}_{(X, Y)\sim {\mathcal{P}}}\left[(f_{\theta}(X), Y) \in R\right]  \leq& \notag \\  &\phantom{aaaaaaaaa}e^{\epsilon} \mathbb{P}_{(X, Y) \sim D} \left[(f_{\theta}(X), Y) \in R\right] + \delta,   \label{def_eq1}\\
&\mathbb{P}_{(X, Y) \sim D} \left[(f_{\theta}(X), Y) \in R\right]    \leq& \notag \\ &\phantom{aaaaaaaaa}  e^{\epsilon}\mathbb{P}_{(X, Y)\sim {\mathcal{P}}}\left[(f_{\theta}(X), Y) \in R\right] + \delta,  \label{def_eq2}
\end{align}
}where  {
\begin{align*}
\mathbb{P}_{(X, Y) \sim D} \left[(f_{\theta}(X), Y) \in R\right]  =  \frac{1}{n}\sum_{i = 1}^n \mathbf{1}\left((f_{\theta}(x_i), y_i) \in R\right),
\end{align*}}
and $f_{\theta}(x)$ is the {prediction} of the trained model {\color{black} $\theta$} with input $x$.
\end{definition}
}
{\color{black} Notice that in our Definition \ref{model_privacy_def},  the randomness comes from sampling either from the training data set $D${\color{black},  or {\color{black}from the population according  to distribution $\mathcal{P}$}, } in contrast to the established DP definition (\ref{e:DP.dwork}), where the randomness is infused by the training mechanism $M$. {\color{black} On the other hand,  $\theta$ and $D$ in Definiton \ref{model_privacy_def} are fixed (not random), as they are simply the model instance and data set observed by the auditor.}}

{\color{black}

}

{\color{black} We parametrize $\tau$ as  $\tau =q_t(\theta)$, a
  quantile of the population loss distribution with value $\tau$.  Therefore, 
\begin{align}
\label{non_train_quantile}
&\mathbb{P}_{(X, Y)\sim {\color{black}\mathcal{P}}}\left[(f_{\theta}(X), Y) \in R\right]  =& \notag \\ &\phantom{aaaaaaaaa}   \mathbb{P}_{(X, Y) \sim {\color{black}\mathcal{P}}}\left[l(f_{\theta}(X), Y) \leq q_t(\theta) \right] =& t,
\end{align}
gives us the false positive rate (FPR) $t$ in the MIA hypothesis
test. We denote by $\eta_t$ the corresponding false negative rate (FNR):
\begin{align}
\label{train_quantile}
&\mathbb{P}_{(X, Y) \sim D} \left[(f_{\theta}(X), Y) \in \bar{R}\right] = &\notag \\ &\phantom{aaaaaaaaa}  \mathbb{P}_{(X, Y) \sim D}\left[l(f_{\theta}(X), Y) > q_t(\theta) \right] = \eta_t.
\end{align}

We vary the threshold $\tau$ by selecting multiple quantiles $t_1, ..., t_k$ in (\ref{non_train_quantile}) corresponding to multiple thresholds $\tau_1,..., \tau_k$, and constructing $k$ rejection regions $R_1(\theta), ..., R_k(\theta)$.  The false positive rate corresponding to rejection region $R_i(\theta)$ is $t_i, i = 1,..., k$, and the corresponding false negative rate we denote $\{\eta_i\}_{i = 1,...,k}$. 

{\color{black}
 In practice, we can estimate $t_1, ..., t_k$ in \eqref{non_train_quantile} from losses obtained from evaluating the model on a non-training set $\bar{D}$ sampled according to $\mathcal{P}$,  and their corresponding $\eta_i$ in \eqref{train_quantile} from losses obtained by evaluating the model on the training set $D$.  One way of estimating $(t, \eta_t)$ is by evaluating the empirical cumulative distribution functions (eCDFs) of the training and non-training losses.  In this case, even though we can select as many thresholds ($\tau_k$) as we want,  we can have at most $|D|$ unique values $\eta_t$.  An alternative,  and better way of estimating $(t, \eta_t)$ is to fit parametric distributions to losses on the population and training data and evaluate the cumulative distribution functions of those fitted distributions. This allows the false positive and false negative rates to be evaluated at a potentially infinite number of points $k$, each resulting in a potentially unique pair of $(t, \eta_t)$. We show in Section \ref{subsec:param} and Appendix B that parametric fitting to loss distributions can greatly improve the accuracy of estimating $(t, \eta_t)$.}

We are now ready to formally define $\epsilon^*$,  which we prove in Appendix A (Section \ref{subsec:appendix_lower_bound_for_def1}) to be a valid lower bound for the $\epsilon$ in Definition \ref{model_privacy_def}.

\begin{definition}[Epsilon*]\label{def_eps_star} Let $\delta$ be a small non-negative real number,  and $\{(t_i, \eta_i)\}_{i = 1,..., k}$ be the false positive and false negative rates, respectively,  corresponding to $k$ rejection regions $\{R_i(\theta)\}_{i = 1,..., k}$ of the form (\ref{R_threshold}),  where each $t_i$ and  $\eta_i$ is described by Equations (\ref{non_train_quantile}) and (\ref{train_quantile}) respectively.  We call the following quantity Epsilon*:
\begin{equation}
\begin{split}
\epsilon^* =& \log \Big[ \max_{i = 1, ..., k} \max \big(\frac{1-\delta - \eta_i}{t_i}, \frac{1-\delta-t_i}{\eta_i},  \\ & \phantom{aaaa} \frac{\eta_i-\delta}{1-t_i}, \frac{t_i-\delta}{1-\eta_i}, {\color{black}1}\big)\Big].
\label{eps_star_levels}
\end{split}
\end{equation}
\end{definition}
}
   {\color{black}The $\delta$ in Definition \ref{def_eps_star} should be chosen as the same $\delta$ in Definition \ref{model_privacy_def},  as $\epsilon^*$ is a lower bound for $\epsilon$ in Definition \ref{model_privacy_def} given the same $\delta$.  Just as in the original DP definition from \cite{dwork2006calibrating}, the $\delta$ in our Definition \ref{model_privacy_def} controls the amount of relaxation in how much the probabilities can differ in Definition \ref{model_privacy_def}.  The general recommendation in the DP literature is to set $\delta \ll 1/n$, where $n$ is the number of records in the training set \cite{ponomareva2023dp},  as this implies that each record in the training set has a worst case probability $\delta$ of being identified, and if $\delta \ll 1/n$, then the expected number of succesfully identified records is less than 1. We recommend setting the $\delta$ in Epsilon* (and Definition \ref{model_privacy_def}) by the same guidelines}.   {\color{black}
  
  {\color{black} 
Note that as the threshold $\tau = q_t(\theta)$ goes to $-\infty$,  $t$ goes to 0, and $\eta_t$ goes to 1.  Alternatively,  when $\tau = q_t(\theta)$ goes to $\infty$,  $t$ goes to 1 and $\eta_t$ to 0.  Both of these cases (($t = 0$,  $\eta_t$ = 1) and ($t = 1$,  $\eta_t$ = 0)) result in an $\epsilon^*$ value of 0, as the 1 inside the logarithm in \eqref{eps_star_levels} ensures that $\epsilon^*$ is always non-negative,  including when $\delta > 0$. This is consistent with $\delta$ governing the amount by which the right-hand-side of the inequalities in Definition \ref{model_privacy_def} can deviate from the $\epsilon$ guarantee.}

{\color{black}
\subsubsection{Epsilon* vs lower bounds for the privacy loss of the training mechanism}
Epsilon* is a lower bound on the privacy loss of one given model instance, i.e., the $\epsilon$ in our empirical privacy Definition \ref{model_privacy_def} (as shown in Appendix A (Section \ref{subsec:appendix_lower_bound_for_def1})),  and not on the privacy loss of the mechanism that generated that model instance.  But one can, in fact, define a version of Epsilon* for the DP training
mechanism itself (setup described in Appendix A), and obtain a statistically valid lower bound on the $(\epsilon, \delta)$-privacy quantification of the training
mechanism.We call the mechanism-version of Epsilon* $\bar{\epsilon}$, and we prove in Appendix A (Section \ref{appendix:epsilon_bar}) the following result connecting $\epsilon^*$ to the DP-mechanism lower bound $\bar{\epsilon}$:

\begin{thm}\label{eps_star_vs_mech}
 Let $\bar{\epsilon}$ be the mechanism-level lower bound on privacy loss,  and let
  $\epsilon^*=\epsilon^*(\theta)$ be the (random) individual model-level lower bound on privacy loss. Then
\begin{equation}
\bar{\epsilon} \leq \log\left[E_{\theta}\left(e^{\epsilon^*}\right)\right]\label{relationship_eq}
\end{equation}
\end{thm}
 We note that $\epsilon^*$ in Theorem \ref{eps_star_vs_mech} is now random,  as it is a function of the model $\theta$,  which in turn is no longer just an observation of a random mechanism but a random variable itself. The expectation in \eqref{relationship_eq} is taken with respect to this randomness.  
 
 We also note that,  while Epsilon* in \eqref{eps_star_levels} can even be computed for model instances from deterministic training algorithms such as boosted decision trees,   in such deterministc cases, however, we cannot talk about lower bounds for differential privacy mechanisms ($\bar{\epsilon}$),  as those mechanisms require infusing randomness. 
 }
\subsection{Parametric Epsilon*}
\label{subsec:param}
We show in {\color{black}Appendix B} that estimating FNRs and FPRs from the eCDFs of loss data sets can lead to unstable implementations of Epsilon* due to the fact that most relevant FNR and FPR values are usually those very close to 0 and 1, where noise is also likely to become amplified by the maximization in (\ref{eps_star_levels}).  We avoid this instability by fitting parametric Gaussian models to the training and non-training (population) loss distributions,  where we first normalize the losses to [0,1],  and then apply a logit transformation similar to the one used by \cite{Carlini2022membership} (Appendix B). {\color{black} Evaluating the cumulative distribution functions (CDFs) of the fitted parametric distributions to loss data instead of evaluating the empirical CDFs of the training and non-training loss data results in an improved accuracy and sensitivity of Epsilon*, as we show in Appendix B.}

Let $F^{pop}$ and $F^{tr}$ be the cumulative distribution functions {\color{black}for the fitted parametric distributions to} the population and training loss data respectively. Recall that for $0 < t <1$,  we select $q_t = q_t(\theta)$ to be the $t$- quantile of $F^{pop}$ in (\ref{non_train_quantile}), so that $q_t = \left(F^{pop}\right)^{-1}(t)$. The FPR is $t$ by equation (\ref{non_train_quantile}), and the corresponding FNR is 
\begin{equation}
\eta_t = 1 - F^{tr}(q_t).
\end{equation}
If $F^{pop}$ and $F^{tr}$ are parametric distributions,  we can obtain a higher lower bound than in (\ref{eps_star_levels}) by letting $k \rightarrow \infty$ and taking supremum instead of maximum in \eqref{eps_star_levels},  thus computing Epsilon* as
\begin{align*}
\epsilon^* &= \log \Big[ \sup_{0<t<1} \max \big(\frac{1-\delta - \eta_t}{t}, \frac{1-\delta-t}{\eta_t},  \\& {\phantom{aaaaa}} \frac{\eta_t-\delta}{1-t}, \frac{t-\delta}{1-\eta_t},1\big)\Big]\\
&= \log \Big[ \max \big( \sup_{0<t<1}\frac{1-\delta - \eta_t}{t},   \sup_{0<t<1}\frac{1-\delta-t}{\eta_t},  \\ &{\phantom{aaaaa}} \sup_{0<t<1}\frac{\eta_t-\delta}{1-t},  \sup_{0<t<1}\frac{t-\delta}{1-\eta_t},1\big)\Big]\\
& =: \log\left[\max(m_1, m_2,m_3,m_4,1)\right]. 
\end{align*}
The values $m_1, m_2,m_3,m_4$ can be found numerically.\footnote{ We observed in our implementation that for small values of $\delta$, the maximum of an {\color{black}$m_i $ for $i =1, ...,4$} can occur  at $t$ very close to 0 or very close 1, which can lead to floating point errors. In our implementation,  we avoid these errors by considering FNR and FPR values between $(\delta, 1-\delta)$.  Another approach to avoiding such numerical errors could be to consider Taylor-like approximations for the fractions in (\ref{eps_star_levels}).}  In particular, each fraction will have one limit (as $t$ goes to 0 or 1) $-\infty$ and the other limit $1-\delta$,  which means its supremum is either $1-\delta$, or the value at {\color{black} one of its local maxima}, whichever is greater.


\section{Experiments}
\label{sec:experiments}

\subsection{Setup}
\label{subsec:setup}
We implemented the parametric version of Epsilon* and evaluated its values on two types of ML deep neural net (DNN) models trained on public data sets: UCI Adult\footnote{http://archive.ics.uci.edu/ml} and Purchase-100\footnote{https://www.kaggle.com/c/acquire-valued-shoppers-challenge/data}.  The UCI Adult data set has about 48,000 data points and predicts the likelihood of someone having an income above \$50,000 based on demographic features obtained from the US census such as age  and salary.  It was split into a training set $D$ of 32,561 rows and a non-training (population) set of 16,280 rows.  We build a model to train on this data set following previously reported benchmark DNN architectures \cite{shokri2017membership}.  

The Purchase-100 dataset consists of roughly 600 binary features and 300,000 records,  with a multi-class vector output of dimension 100. We randomly split the data set into a training set of size 208,731 and a non-training (population) set of size 102,809,  then created an overly overfitted DNN model with three layers, each of 100 nodes, trained until the training accuracy was 100\%.  {\color{black}Over-fitting has been associated with poor privacy \cite{yeom2018privacy}, as it results in a training output distribution that is different from the non-training output distribution.}

\subsubsection{Privacy strategies}
\label{subsubsec:strategies}
We considered several baseline models corresponding to training for different number of epochs  (1, 5, 10,  20 epochs for Adult,  and 10, 50, 100 for Purchase-100), as well as training with DP-SGD  until reaching a DP $\epsilon$ value of 1, 3, 10 and 100, where $\delta$ was set to $1/n\log n$, for $n$ = training set size.   As showcased by other work \cite{denison2022private, ponomareva2023dp},  achieving good model utility when training with DP-SGD is non-trivial, and requires at the minimum a search through several hyperparameter values such as batch size, learning rate,  clipping norm, and number of microbatches.  {\color{black} We grid searched through a few hyperparameter values inspired by recent work describing hyperparameter tuning for DP-SGD \cite{denison2022private}}: we used a SGD optimizer with a momentum value of 0.9 and learning rate of 0.01, and varied the clipping norm in the set \{1, 10,  30\}, the batch size in the set \{64,  512\} for Adult and \{256, 1024\} for Purchase-100, and the number of micro-batches in the set \{1, 32\}. We reached a desired DP epsilon value by performing a binary search on the noise multiplier parameter when all other parameter were set.  This resulted in 51 experiments (hyperparameter -- privacy strategy combinations) for Purchase-100 and 52 experiments for Adult.

In performing the experiments, we fixed the training and non-training data sets (i.e.  we do not resample), and repeated the experiments for each strategy five times to observe variability due to training with SGD, resulting in more than 250 model instances for each data set.  {\color{black} Note that, in plotting the privacy-utility landscape,  we consider the DP parameter $\epsilon$ used in training as part of the hyperparameter set defining a \textit{privacy mitigation} strategy,  along with the other hyperparameters (batch size, clip norm, etc.).  In contrast,  Epsilon* describes a \textit{measurement} for that strategy's privacy risk.} We measure the {\color{black}strategy's utility using} area under the receiver operating characteristic (AUROC) curve on the non-training set for the Adult models, and accuracy on the test set for Purchase-100 due to the multi-class objective.  {\color{black}For training with DP, we used the TensorFlow Privacy library\footnote{https://www.tensorflow.org/responsible\_ai/privacy/guide} (v.0.5.1), which utilizes a R\'enyi differential privacy (RDP){\color{black}\cite{mironov2017renyi}} accountant function in order to estimate the DP parameter $\epsilon$ used when training.}

\subsection{Results}
\label{subsec:Results}
To understand the privacy-utility trade-offs, we adopt a multi-objective decision-making framework \cite{roijers2017multi} where we show a privacy-utility plot for each data set and highlight the Pareto frontier on each plot.  {\color{black}The Pareto frontier  of a given set of points in the privacy-utility plane is the convex hull of strategies in this plane,  taken in a direction that improves the privacy or the utility, i.e.,  any strategy on the Pareto frontier has the property that there are no two other strategies whose convex combination results in a point with higher utility or higher privacy. } 

We display values of utility and Epsilon* for the public data set models.  {\color{black} For each privacy-utility plot,  in addition to highlighting the Pareto frontier,  we also show marginal densities (using kernel density estimation) of the categories of strategies of the same type,  e.g.,  marginal density of Epsilon* for strategies where we train with DP versus baseline strategies where we do not, in order to visualize the relative values of the metrics for each category of strategies.}

\subsubsection{Purchase-100}
\begin{figure*}[h!]
  \centering
  \includegraphics[width=0.74\linewidth]{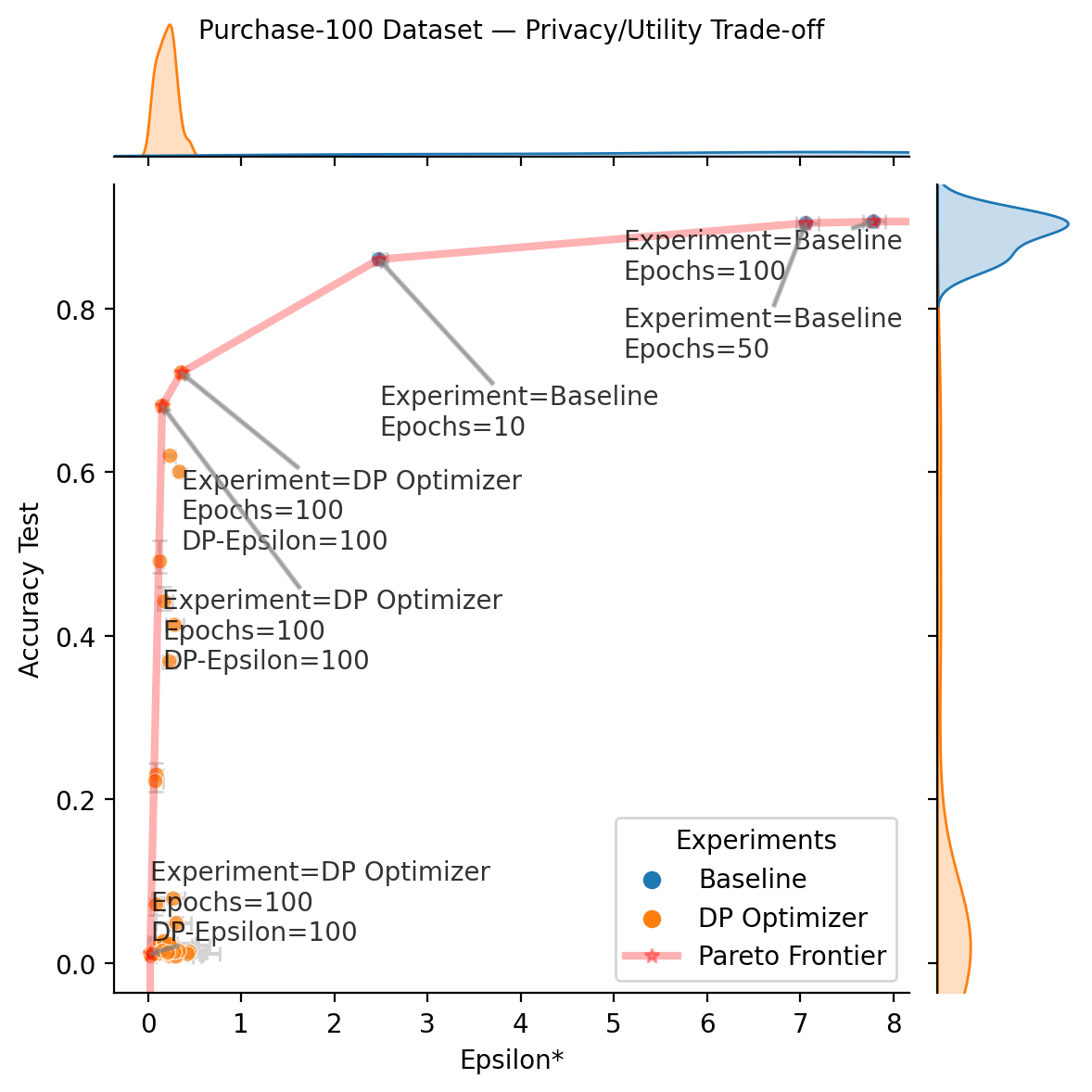}
  \caption{Privacy-Utility trade-off for the Purchase-100 data set.  Each point correspons to one strategy (same hyperparameter set and privacy mitigation),  averaged over five model instances.  Error bars correspond to the minimum and maximum values over the five model instances. }
  \label{fig:purchase}
\end{figure*}
Figure \ref{fig:purchase} displays the privacy-utility plot for the Purchase-100 data set.  Each point is the average of five experiments with the same hyperparameters and DP training $\epsilon$ (if applicable),  and therefore even if we trained to the same DP $\epsilon$,  we may still have different utility and Epsilon* values depending on the set of hyperparameters used.  {\color{black} For instance,  we observe on the Pareto frontier of Figure \ref{fig:purchase} multiple strategies trained with  DP to the same $\epsilon$ (in this case 100),  but different hyperparameter sets.  Notice that, even though the  Epsilon* values for these three sets of experiments are relatively similar (and all less than 1), the test accuracy values are significantly different,  varying from 1\% to 72\%. } Similarly, we have multiple points for baseline models,  corresponding to training for different number of epochs. 

Baseline experiments that only vary the number of training epochs show that overfitting is correlated with poor privacy metrics, and provide a quantifiable value for the privacy-utility trade-off. {\color{black} We observe that increasing the number of training epochs increases the test accuracy of the model: 86\% test accuracy when training for 10 epochs, to 90.4\% when we train for 50, and 90.6\% when we train for 100 epochs, but also increases Epsilon* from 2.48 at 10 epochs to 7.06 at 50 epochs, and then 7.79 at 100 epochs.  This behavior is consistent with the intuition that the more epochs we train, the more the model overfits, which results in higher privacy risk metric values.}

{\color{black} When training with DP,  Epsilon* values become significantly smaller:} the highest Epsilon* experiments reach values of 7.8 for baseline models tuned with different hyperparameters, while all experiments where we train with DP had an Epsilon* value below 1.  This is also visible from the marginal density of Epsilon* on the top of Figure \ref{fig:purchase}, which is concentrated around values closer to 0 than to 1.  However, 
DP has a sizable impact on model utility, as the marginal densities of accuracy values with and without DP,  displayed on the right of the plot, also show. This is an example of the complexities of using DP in real-world scenarios that are commonly discussed in the literature as well (e.g. Fig. 2, Section 5.4 in \cite{ponomareva2023dp}).

\subsubsection{Adult}
\begin{figure*}[h!]
  \centering
  \includegraphics[width=0.74\linewidth]{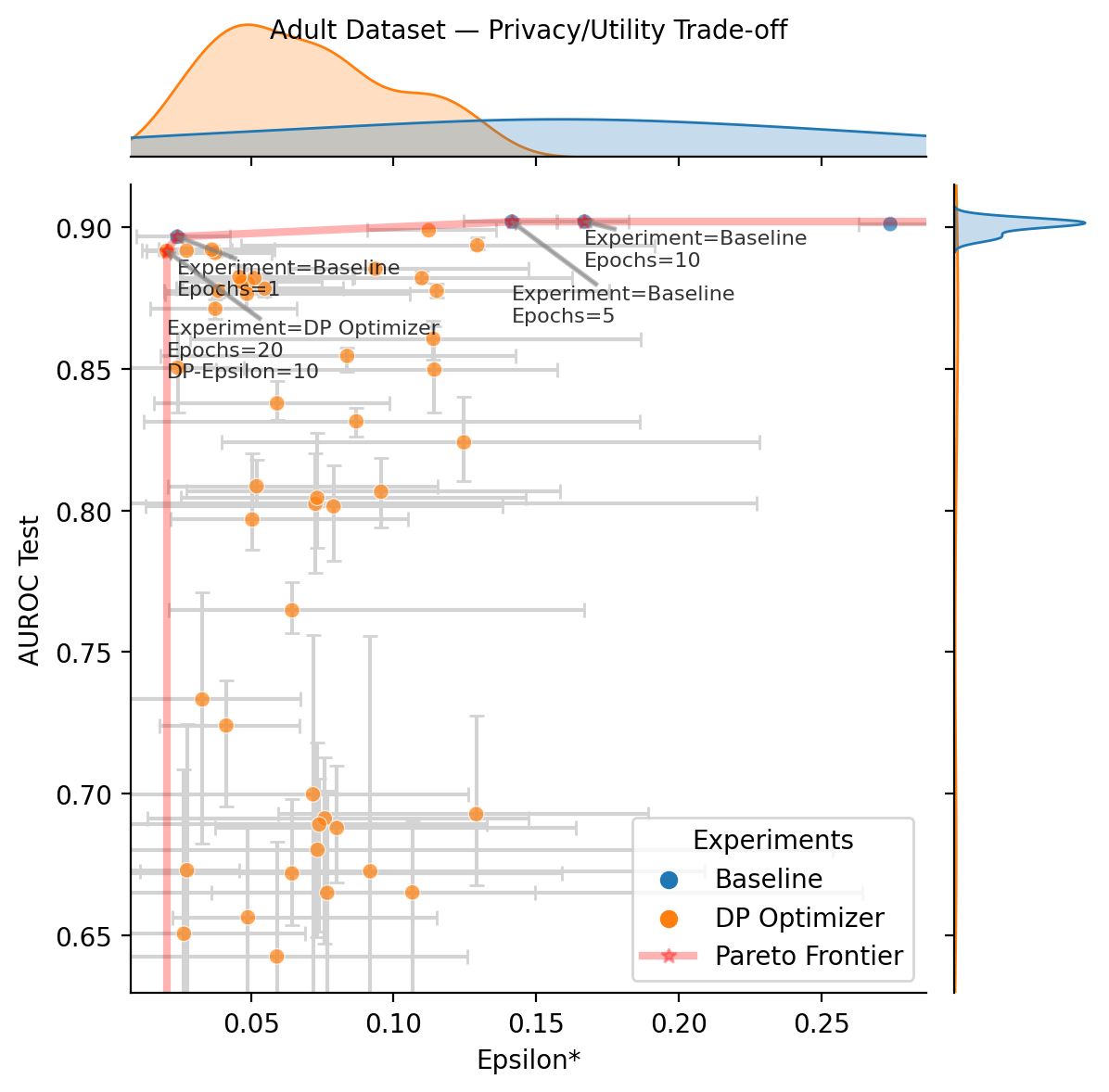}
  \caption{Privacy-Utility trade-off for the Adult data set.  Each point correspons to one strategy (same hyperparameter set and privacy mitigation),  averaged over five model instances.  Error bars correspond to the minimum and maximum values over the five model instances. }
  \label{fig:adult}
\end{figure*}
Figure \ref{fig:adult} shows the privacy-utility landscape for the Adult data set.  We note that all our model instances resulted in Epsilon* values below 0.3, which in comparison to Purchase-100, do not indicate privacy risk for this model across experiments and hyperparameters. Even so,  we observe an increase in Epsilon* as the number of epochs is increased for the baseline experiments, {\color{black} and model instances trained with DP having lower values of Epsilon* than baseline (non-DP) model instances trained for 5 or more epochs}. 
The “elbow” model instance in this Pareto frontier corresponds to a baseline model trained for a single epoch,  which nonetheless outperforms most DP models as they lay below the Pareto frontier.  We also note that even though the {\color{black} min-max} ranges of Epsilon* for the Adult experiments {\color{black}(error bars in Figure \ref{fig:adult})} appear to be larger than those for the Purchase-100 experiments,  the absolute values of the averages are much smaller for Adult {\color{black} (less than 0.3)} than for Purchase-100 {\color{black} (up to 7.8)},  indicating that the latter models likely exhibit generally higher privacy risk than the former. 


\section{Conclusion}
\label{sec:conclusion}
We introduced Epsilon*, a new privacy metric for measuring the privacy of a single model instance {\color{black} that enables privacy auditors to be independent of model owners} {\color{black}by requiring only black-box access to model predictions.  This indepenendence is very important in industry settings,  where the teams charged with productionalizing the models and privacy auditors may be distinct,  the latter not having model-specific domain knowledge and being tasked to measure privacy risk at scale,  for a wide range of use cases.  

We show how Epsilon* can be derived from a model-level  {quantification of} privacy, which we call \textit{empirical privacy}, and to which the hypothesis test formulation of DP can be extended.  
We note that our goal in defining empirical privacy is not to re-invent differential privacy -- firstly, we aim to \textit{measure}, not mitigate, model-level privacy risk, and secondly,  DP is already established as the gold standard mitigation strategy of privacy risk for training mechanisms, which  naturally extends to the model instances of those mechanisms.  
Our task instead was motivated by the industry-wide need of measuring the privacy risk of large model instances,  \textit{regardless of whether DP was used in training}, and without having to insert canaries,  re-sample the training set,  or re-train the model.  Epsilon*, the metric derived from our privacy definition,  satisfies these design requirements. }

We show in experiments on benchmark public data sets that Epsilon* is sensitive to privacy mitigation with DP. We also establish a relationship between Epsilon* and a lower bound on {\color{black}the privacy loss of} the training mechanism, and show how to implement Epsilon* to avoid numerical and noise amplification instability.   

\section{Acknowledgments}
The authors would like to thank Ryan Rogers,  Rina Friedberg,  Mark Yang,  and the anonymous reviewers for their constructive feedback that greatly improved the clarity of this manuscript.

\bibliography{Epsilon_star_reference.bib}

\newpage

\clearpage

\section{Appendix A: Proofs}
\label{sec:appendix_A}

\subsection{Epsilon* is a valid lower bound for $\epsilon$ in Definition \ref{model_privacy_def}}
\label{subsec:appendix_lower_bound_for_def1}
We provide the following analog  of Theorem \ref{Kairouzhm} in our empirical,  model-level
privacy quantification (Definition
\ref{model_privacy_def}):
\begin{thm}\label{Kairouzhm_for_model}
A trained model  instance parameterized by $\theta$,  and obtained from using a training mechanism M on data set $D = \{(x_i, y_i), i = 1, ..., n\} $,  is $(\epsilon,  \delta)$-private if and only if for all rejection regions $R$ of the form \eqref{R_threshold}:
\begin{align}
&\mathbb{P}_{(X, Y)\sim {\color{black}\mathcal{P}}}\left[(f_{\theta}(X), Y) \in R\right] +& \notag \\ & \phantom{aaaaa} e^{\epsilon} \mathbb{P}_{(X, Y) \sim D} \left[(f_{\theta}(X), Y) \in \bar{R}\right]  \leq    e^{\epsilon} + \delta \label{m1}\\
 &\mathbb{P}_{(X, Y) \sim D} \left[(f_{\theta}(X), Y) \in \bar{R}\right] + & \notag  \\ & \phantom{aaaaa} e^{\epsilon} \mathbb{P}_{(X, Y)\sim {\color{black}\mathcal{P}}}\left[(f_{\theta}(X), Y) \in R\right] \geq 1-   \delta \label{m2}\\
 &\mathbb{P}_{(X, Y)\sim {\color{black}\mathcal{P}}}\left[(f_{\theta}(X), Y) \in R\right] + & \notag \\ & \phantom{aaaaa} e^{\epsilon}  \mathbb{P}_{(X, Y) \sim D} \left[(f_{\theta}(X), Y) \in \bar{R}\right] \geq  1-   \delta \label{m3}\\
 &\mathbb{P}_{(X, Y) \sim D} \left[(f_{\theta}(X), Y) \in \bar{R}\right] + & \notag \\ & \phantom{aaaaa} e^{\epsilon}\mathbb{P}_{(X, Y)\sim {\color{black}\mathcal{P}}}\left[(f_{\theta}(X), Y) \in R\right]  \leq    e^{\epsilon} + \delta \label{m4}
\end{align}
where $f_{\theta}(x)$ is the {\color{black}prediction} of the trained model {\color{black} $\theta$} with input $x$.
\end{thm}

The proof of Theorem \ref{Kairouzhm_for_model} follows easily from Definition \ref{model_privacy_def}, just like Theorem \ref{Kairouzhm} follows easily from the DP definition in Equation (\ref{e:DP.dwork}): (\ref{m1}) and (\ref{m2}) are directly from (\ref{def_eq1}) and (\ref{def_eq2}) respectively,  and (\ref{m3}) and (\ref{m4}) follow from (\ref{def_eq1}) and (\ref{def_eq2}) by replacing $R$ with $\bar{R}$.

We can obtain a valid lower bound
on $\epsilon$ in  Definition \ref{model_privacy_def} by utilizing Theorem \ref{Kairouzhm_for_model},  and exploiting the intuition that any $\epsilon$ value that is smaller than the minimum value satisfying Eq. (\ref{m1} - \ref{m4}) is a counter-example to the model being $(\epsilon, \delta)$-private as in Definition \ref{model_privacy_def}.  Therefore,  a valid lower bound on the resilience of that model instance can be obtained by solving the following linear program, where the constraints come from Eq. (\ref{m1} - \ref{m4}) with the FPR and TPR notation of (\ref{non_train_quantile}) and (\ref{train_quantile}):
\begin{align*}
\text{minimize} \phantom{aa} &\epsilon \\
\text{subject to}  \phantom{aa}&t + e^{\epsilon} \eta_t  \leq   e^{\epsilon} + \delta \\
 &\eta_t + e^{\epsilon}t  \geq 1-   \delta \\
 &t + e^{\epsilon} \eta_t \geq 1-   \delta \\
 &\eta_t + e^{\epsilon}t \leq   e^{\epsilon} + \delta.
\end{align*}
To improve this lower bound (i.e., find a higher one),  we select 
multiple quantiles $t_1$, ..., $t_k$ in (\ref{non_train_quantile}) and
construct $k$ different rejection regions corresponding to these
levels,  $R_1(\theta)$,  ..., $R_k(\theta)$.  By definition
(\ref{non_train_quantile}), the false positive rate corresponding to
the rejection region $R_i(\theta)$ is exactly $t_i$, $i = 1, .., k$,
and we let $\eta_i$ be the false negative rate corresponding to the
rejection region $R_i(\theta)$. Then, each line in the linear program above will describe $k$ constraints corresponding to the $k$ rejection regions,  and its solution, $\epsilon^*$ in {\color{black} Definition} \ref{def_eps_star},  is a valid lower bound for $\epsilon$ in Definition
\ref{model_privacy_def} by Theorem \ref{Kairouzhm_for_model}. 

The best lower bound of this
  type is obtained when we {\color{black}choose the 
  quantiles $(t_i)$ densely,  i.e.  let $k\to\infty$}. This observations turns out to be very
  useful in the parametric Epsilon* described in {\color{black}Section \ref{subsec:param} and Appendix B}. 
  
 {\color{black}\subsection{Epsilon* vs lower bounds for the privacy loss of the training mechanism} \label{appendix:epsilon_bar}}
 Let us now define a version of Epsilon* for the training
mechanism itself,  which will serve as a statistically valid lower} {\color{black}bound, say
$\bar{\epsilon}$,  on the $(\epsilon, \delta)$-privacy quantification of the training
mechanism.  Recall that, generally, 
$\theta$ in Definition \ref{model_privacy_def} is an observation from a random variable -- obtained by applying random training mechanism $M$ to data set $D$. Correspondingly, we 
view the lower bound $\epsilon^*$ corresponding to a model
instance as an observation from a random variable (a function of the training mechanism). }{\color{black}Under an
appropriate setup, there is a connection between the mechanism-level 
$\bar{\epsilon}$ and the model-level random $\epsilon^*$. {\color{black}We now provide
the setup and the connection between these quantities.}
}

Suppose that instead of measuring the resilience of a trained model we
now wish to measure the resilience of the training mechanism.
{\color{black}   Consider the probabilities on the left sides of the inequalities in
  Definition \ref{model_privacy_def} for any choice of {\color{black}(measurable)} rejection regions $R(\theta)$ of
the form (\ref{R_threshold}),  with additional averaging over
  choices of the training sets and randomness of the training mechanism:}
{\color{black}
\begin{align}
&  E_{D,\theta}\left[\mathbb{P}_{(X,Y) \sim {\color{black}\mathcal{P}}}\left[(f_{\theta}(X), Y) \in R(\theta) \right]\right] \leq \notag \\ & \phantom{aaaaaa} e^{\epsilon} E_{D,\theta}\left[\mathbb{P}_{(X,Y) \sim D}[(f_{\theta}(X), Y) \in R(\theta)]\right] + \delta,  \label{mechanism_privacy_0}\\
 &E_{D,\theta}\left[\mathbb{P}_{(X,Y) \sim D}[(f_{\theta}(X), Y) \in R(\theta)]\right]  \leq \notag \\ & \phantom{aaaaaa}  e^{\epsilon}E_{D,\theta}\left[\mathbb{P}_{(X,Y) \sim {\color{black}\mathcal{P}}}\left[(f_{\theta}(X), Y) \in R(\theta) \right]\right] + \delta,
 \label{mechanism_privacy}
\end{align}

where $E_{D,\theta}$ is the expectation taken with respect to the randomness
producing training set $D$ and the randomness infused by the training
mechanism{\color{black}, where $\theta = M(D)$}, {\color{black}$(X, Y) \sim \mathcal{P}$ refers to sampling according to distribution $\mathcal{P}$,  {\color{black} and $\mathbb{P}_{(X, Y) \sim D} \left[(f_{\theta}(X), Y) \in R\right]  =  \frac{1}{n}\sum_{i = 1}^n \mathbf{1}\left((f_{\theta}(x_i), y_i) \in R\right)$ as in Definition \ref{model_privacy_def}.}

The following result states that differential privacy implies \eqref{mechanism_privacy_0} and \eqref{mechanism_privacy}, and therefore these two equations can be used to derive a lower bound on the privacy loss of DP, which we call $\bar{\epsilon}$,  similar to how we used our Definition \ref{model_privacy_def} to derive Epsilon*.
\\
\begin{thm}\label{epsilon_bar_vs_epsilon}
Let $M$ be an ($\epsilon$, $\delta$)-differentially private training mechanism that generates model instances parameterized by $\theta$ given data sets $D$,  and let ${\color{black}\mathcal{P}}$ be {\color{black} the distributon over the data population.  }

Then,  {\color{black}the smallest} $\bar{\epsilon}$ that satisfies (\ref{mechanism_privacy_0}) and (\ref{mechanism_privacy}) for any choice of rejection region $R(\theta)$ of
the form (\ref{R_threshold}), is a lower bound for the $\epsilon$  parameter of $M$ for the same value of $\delta$.
\end{thm}

\begin{proof}
{\color{black}If mechanism $M$ is $(\epsilon,  \delta)$-DP,  per \eqref{e:DP.dwork},  then if we fix $(x,y) \in {\color{black}\mathcal{X} \times \mathcal{Y}}$, we must have that for any two neighboring data sets $D$ and $D'$
\begin{align}
&\mathbb{P}_{\theta \sim M(D)}\left[(f_{\theta}(x), y) \in R(\theta)\right] \leq \notag \\ & \phantom{aaaaaa} e^{\epsilon}\mathbb{P}_{\theta \sim M(D')}\left[(f_{\theta}(x), y) \in R(\theta)\right] + \delta,\label{dp_1}
\end{align}
where on the left side the probability is  with respect to the randomness of the training mechanism  using $D$ as training set, and on the right using $D'$. Note that \eqref{dp_1} holds regardles of the choice of the set $R(\theta)$, assuming a measurability of the mapping $\theta \rightarrow R(\theta)$.
\eqref{dp_1} also holds whether or not $(x,y)$ is in $D$ and/or $D'$. 

 If $D = \{(x_1,y_1), ..., (x_n,y_n)\}$, then for each $i = 1,...,n$ we can use \eqref{dp_1} for the pair $(x_i, y_i)$ and any set {\color{black} $D'_i = D\setminus \{(x_i, y_i)\} \cup \{(x_i^*,y_i^*)\}$, where $(x_i^*,y_i^*)$ is sampled from $\mathcal{P}$ conditionally on not being in $D$.}  Summing up the resulting inequality over all points in $D$ and averaging,  gives us
{\color{black}
\begin{align}
&\frac{1}{n}\sum_{i=1}^n\mathbb{P}_{\theta \sim M(D)}\left[(f_{\theta}(x_i),y_i) \in \R(\theta)\right] \leq \notag \\ & \phantom{aaaaaa} e^{\epsilon}\frac{1}{n}\sum_{i=1}^n\mathbb{P}_{\theta \sim M(D_i')}\left[(f_{\theta}(x_i),y_i) \in R(\theta)\right] + \delta. \label{dp_3}
\end{align}
}
We can now {\color{black}take expectation in} \eqref{dp_3} over all training sets $D$.  {\color{black} The left side of \eqref{dp_3} averages to the left side of \eqref{mechanism_privacy}, with different notation.  The right side of \eqref{dp_3} averages to the right side of \eqref{mechanism_privacy}. {\color{black}Therefore,} we obtain} \eqref{mechanism_privacy}
\begin{align*}
& E_{D, \theta}\left[\mathbb{P}_{(X,Y) \sim D}\left[(f_{\theta}(X), Y) \in R(\theta)\right]\right] \leq \notag \\ & \phantom{aaaaaa} e^{\epsilon}E_{D,\theta}\left[\mathbb{P}_{(X,Y) \sim {\color{black}\mathcal{P}}}\left[(f_{\theta}(X), Y) \in R(\theta) \right]\right] + \delta.
\end{align*}
 \eqref{mechanism_privacy_0} can be obtained analogously by switching $D$ with $D'$ in \eqref{dp_1}.  This implies that any $\epsilon$ that breaks either  \eqref{mechanism_privacy_0} or  \eqref{mechanism_privacy} cannot satisfy the DP property \eqref{dp_1}, so the smallest $\bar{\epsilon}$ that satisfies both is a lower bound for the $\epsilon$ of $M$.}
\end{proof}

We now find an expression for $\bar{\epsilon}$ along the same lines as Theorem \ref{Kairouzhm} (recently also used by \cite{zanella2022bayesian}).
Suppose we choose a rejection region $R(\theta)$ of
the form (\ref{R_threshold}) with quantile $q_t(\theta)$ still chosen
according to (\ref{non_train_quantile}).  

Fixing a $t$,  recall that by (\ref{non_train_quantile}),
\begin{align*}
&\mathbb{P}_{(X, Y)\sim {\color{black}\mathcal{P}}}\left[(f_{\theta}(X), Y) \in R\right]  = \\& \phantom{aaaaaa}  \mathbb{P}_{(X, Y) \sim {\color{black}\mathcal{P}}}\left[l(f_{\theta}(X), Y) \leq q_t(\theta) \right] = t,
\end{align*} for any trained weights $\theta$. The false positive rate for the mechanism-level hypothesis test is then
\begin{align*}
FPR = E_{D,\theta}\left[\mathbb{P}_{(X,Y) \sim {\color{black}\mathcal{P}}}\left[(f_{\theta}(X), Y) \in R(\theta) \right]\right] = t.
\end{align*}
Note that the corresponding mechanism-level false negative rate, FNR, is 
\begin{align}
&\eta_t = FNR = E_{D, \theta}\left[\frac{1}{n}\sum_{i=1}^n \mathbf{1}(l(f_{\theta}{\color{black}(X_i), Y_i}) > q_t(\theta))\right] := \notag \\ & \phantom{aaaaaa} E_{D, \theta}\left(\eta_t (\theta)\right),
\label{fnr_mechanism}
\end{align}
where we view $\eta_t(\theta)$ as the FNR corresponding to a fixed
model $\theta$,  {\color{black} resulting from training mechanism $M$ on data set $D$,}  $l(f_{\theta}(X_i), Y_i)$ is the loss of the model
$\theta$ at point $(X_i, Y_i)$, and $D = \{(X_i, Y_i)_{i = 1, ..., n}\}$ is a sampled training set.

A lower bound on the privacy loss of a differentially private mechanism is then
\begin{equation}
\begin{split}
\bar{\epsilon} =& \log \Big[ \max_{i = 1, ..., k} \max \big(\frac{1-\delta - \eta_{t_i}}{t_i}, \frac{1-\delta-t_i}{\eta_{t_i}},  \\ & \phantom{aaaa} \frac{\eta_{t_i}-\delta}{1-t_i}, \frac{t_i-\delta}{1-\eta_{t_i}}, 1\big)\Big]
\label{eps_star_mechanism}
\end{split}
\end{equation}
with $\eta_{t_i}$ given by (\ref{fnr_mechanism}).

On the other hand,  Epsilon* for a given model $\theta$  is random,
and given by 
\begin{equation}
\begin{split}
\epsilon^*=\epsilon^*(\theta) = &  \log \Big[ \max_{i = 1, ..., k} \max \big(\frac{1-\delta - \eta_{t_i}(\theta)}{t_i},  \\ & \phantom{aaaa} \frac{1-\delta-t_i}{\eta_{t_i}(\theta)}, \frac{\eta_{t_i}(\theta)-\delta}{1-t_i}, \frac{t_i-\delta}{1-\eta_{t_i}(\theta)}, 1\big)\Big].
\label{eps_star_theta}
\end{split}
\end{equation}

With this setup, we have  
the following result connecting the random $\epsilon^*$ with the lower bound
$\bar{\epsilon}$ on the $(\epsilon, \delta)$-differential privacy of a training
mechanism (\textbf{Theorem \ref{eps_star_vs_mech}}):  
\begin{equation}
\bar{\epsilon} \leq \log\left[E_{\theta}\left(e^{\epsilon^*}\right)\right]
\end{equation}
}

\begin{proof}[Proof of Theorem \ref{eps_star_vs_mech}]
Notice that for each $t > 0$, the functions $\frac{1-\delta - \eta}{t}$,  $\frac{(1-\delta-t)_+}{\eta}$,  $\frac{\eta_-\delta}{1-t}$, $\frac{(t-\delta)_+}{1-\eta}$ are all convex functions of $0< \eta <1$.  Furthermore,  the maximum of convex functions is itself a convex function. Therefore,
 \begin{equation*}
 \begin{split}
& \max_{i = 1, ..., k} \max \big(\frac{1-\delta - \eta_{t_i}(\theta)}{t_i}, \frac{1-\delta-t_i}{\eta_{t_i}(\theta)}, \\ & \phantom{aaaaaaaaaaa} \frac{\eta_{t_i}(\theta)-\delta}{1-t_i}, \frac{t_i-\delta}{1-\eta_{t_i}(\theta)}, 1\big)
 \end{split}
\end{equation*}
is a convex function of the random vector $(\eta_{t_i}(\theta),  i  =
1, ...,  k)$. {\color{black} Taking expectation with respect to the
  randomness of the model weights $\theta$ gives us, by Jensen's inequality,}
 \begin{align*}
E\left(e^{\epsilon^*(\theta)}\right) &= E\Big[ \max_{i = 1, ..., k} \max \Big(\frac{1-\delta - \eta_{t_i}(\theta)}{t_i},  \\ & \phantom{aaaa} \frac{1-\delta-t_i}{\eta_{t_i}(\theta)}, \frac{\eta_{t_i}(\theta)-\delta}{1-t_i}, \frac{t_i-\delta}{1-\eta_{t_i}(\theta)}, 1\Big)\Big]\\
&\geq  \max_{i = 1, ..., k} \max \Big(\frac{1-\delta - E[\eta_{t_i}(\theta)]}{t_i},  \\ & \phantom{aaaa} \frac{1-\delta-t_i}{E[\eta_{t_i}(\theta)]}, \frac{E[\eta_{t_i}(\theta)]-\delta}{1-t_i}, \frac{t_i-\delta}{1-E[\eta_{t_i}(\theta)]}, 1\Big)\\
&= e^{\bar{\epsilon}}
\end{align*}
where we used (\ref{fnr_mechanism}).  We conclude that
\begin{equation*}
\bar{\epsilon} \leq \log\left[E_{\theta}\left(e^{\epsilon^*}\right)\right] 
\end{equation*}
\end{proof}

\section{Appendix B: Parametric Epsilon*}
\label{sec:appendix_B}
Recall from Equation (\ref{eps_star_levels}) that Epsilon* is computed from the estimated FNRs $\eta_i$ and FPRs $t_i$ of the hypothesis test of the MIA adversary.  
\begin{equation}
\begin{split}
\epsilon^* =& \log \Big[ \max_{i = 1, ..., k} \max \big(\frac{1-\delta - \eta_i}{t_i}, \frac{1-\delta-t_i}{\eta_i},  \\ & \phantom{aaaa} \frac{\eta_i-\delta}{1-t_i}, \frac{t_i-\delta}{1-\eta_i}, 1\big)\Big]
\end{split}
\end{equation}
Accurate estimates of FNRs and FPRs are therefore crucial to the accuracy of Epsilon*.  In fact,  the main cause for inaccurate values of Epsilon* is  noise amplification when the TPR/FPR becomes close to 0 and 1, which can be additionally inflated by numerical instability from computation involving maxima over logs of (ratios of) very small numbers.

To better visualize the noise impact,  Figure \ref{fig:boundary} shows the values of the first two quantities in (\ref{eps_star_levels})  versus $\frac{1}{2}(TPF + FPR)$ when we simulate loss data from known Gaussian distributions ($\mathcal{N}$(0,1) for training and $\mathcal{N}$(3,1) for non-training) and compute the FNR's and FPRs either by directly evaluating the CDFs of the known loss distributions (a),  or by evaluating the empirical CDFs of simulated losses (b).  In both cases, we remove FNR and FPR values less than $10^{-9}$ for numerical stability.
\begin{figure*}[t!]
    \centering
    \begin{subfigure}[t]{0.5\textwidth}
        \centering
        \includegraphics[height = 2in]{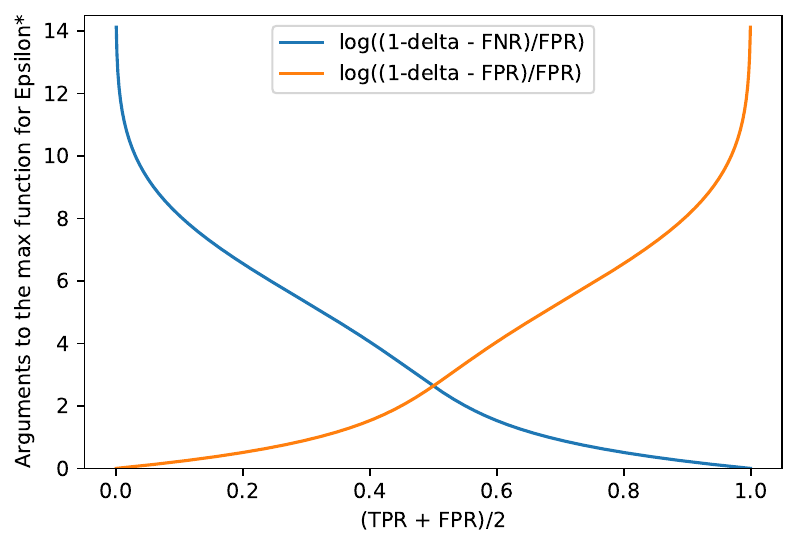}
        \caption{Evaluating the CDFs of the true loss distributions}
    \end{subfigure}%
    ~ 
    \begin{subfigure}[t]{0.5\textwidth}
        \centering
        \includegraphics[height = 2in]{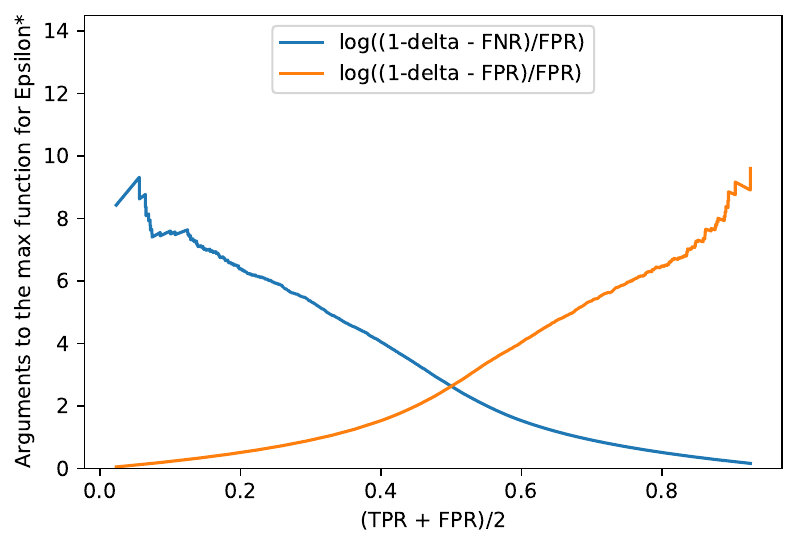}
        \caption{Evaluating the empirical CDFs from samples of the loss distributions}
    \end{subfigure}
    \caption{First two expressions in (\ref{eps_star_levels}), when evaluating the true versus the empirical CDFs from sampled data}
    \label{fig:boundary}
\end{figure*}
As seen in Figure \ref{fig:boundary},  most noise and most signal occurs at the very ends of the FNR/ FPR arrays,  since Epsilon* is the maximum over (noisy) quantities in \ref{eps_star_levels},  and that maximum occurs in the noisiest part of the curves.

\subsection{Simulation of Epsilon* from empirical vs parametric distributions fitted to model losses}
\label{subsec:simulation}
To evaluate the impact of noise amplification and mitigation strategies, we set up a simulation where we decide what the true loss distributions are, and generate loss samples from these distributions. This allows us to compare the Epsilon* values from evaluating the true TPR/FPRs CDFs versus estimating TPR/FPR values from sampled losses empirically.

In our simulation, we mostly focus on Gamma loss distributions, as they have the desired support in $[0, \infty)$. Specifically, we consider the training and non-training set losses to be Gamma distributed with parameters $(k_1, \theta_1)$ and  $(k_2, \theta_2)$ respectively.  Our goal for this simulation is to compare the values of Epsilon* when the $FNRs$ and $FPRs$ used to computed it in Equation (\ref{eps_star_levels}) are obtained by either:
\begin{itemize}
\item directly evaluating the CDFs of the known Gamma distributions used to generate the training and non-training losses; 
\item evaluating the empirical CDFs of the sampled training and non-training losses; or
\item fitting Normal distributions to a transformation of the losses,  then evaluating the CDFs of the fitted parametric distributions.
\end{itemize}
The transformation of losses mentioned in the last bullet point is the following, similar to the one used by [Carlini et al. (2022)]:
\begin{enumerate}
\item Normalize losses to [0,1] by subtracting min(train loss,  non-train loss) from each loss value and dividing by (max(train loss,  non-train loss) - min(train loss,  non-train loss)).
\item Shift the normalized losses from Step 1 away from 0 by adding $\alpha$ (=1).
\item Exponentiate the negative of losses in Step 2:  $p$ = $\exp$(-Step 2 losses).
\item Pass the losses form Step 3 through a logit function: $\phi = \log(p) -\log(1-p)$
\end{enumerate}

Note that evaluating the CDFs of the known generating Gamma distributions is only available in our simulation, as in practice we do not posses knowledge of the true distributions generating the observed losses. In each of the three variants, we evaluate the CDFs at 2 million points correponding to 1 million equally spaced quantiles of the training distribution and 1 million equally spaced quantiles of the non-training distribution.  For numerical stability,  we remove FNR and FPR values outside the interval (0.001, 0.999) for the empirical Epsilon* and outside the interval $(\delta, 1-\delta)$ for the parametric Epsilon*. 

We also inspect the impact of loss set sample size $n$ by varying it in the set $\{10^3, 10^4,  10^5\}$.  To evaluate the noise in our Epsilon* values,  we repeat the sampling (of size $n$) and computation of Epsilon* values $k$ times, where $k = 10$.

We investigate the relative values of the three Epsilon* implementations in two regimes: when the true training and non-training loss distributions are identical, implying that Epsilon* should have a value of 0;  and as the true loss distributions become farther apart. We simulate the latter case by considering the true training loss distribution to be $\Gamma(k_1, \theta_1)$,  and the true non-training loss distribution to be $\Gamma(k_1 + d, \theta_1)$, for $d \in \{1,2,3\}$, $k_1 = 2$, $\theta_1$ = 5. 

\subsubsection{Identical Loss Distributions}
A special case worth examining in our simulation is when the true training and non-training loss distributions are identical. This means that the true Epsilon* is 0, and therefore implementations that get closer to that value are more accurate.  

We show the values of Epsilon* (means and standard deviations over the 10 runs) computed with the empirical CDFs and the parametric fits to the loss distributions in Figure \ref{fig:sim_d0}.  We make two observations based on this figure: 1) The parametric implementations shows values closer to the true value of 0 than the empirical CDF implementation; and 2) there is an impact of the loss data set size: the larger the loss sample size, the closer both implementations get to the true value of 0.
\begin{figure}[h!]
  \centering
  \includegraphics[width=\linewidth]{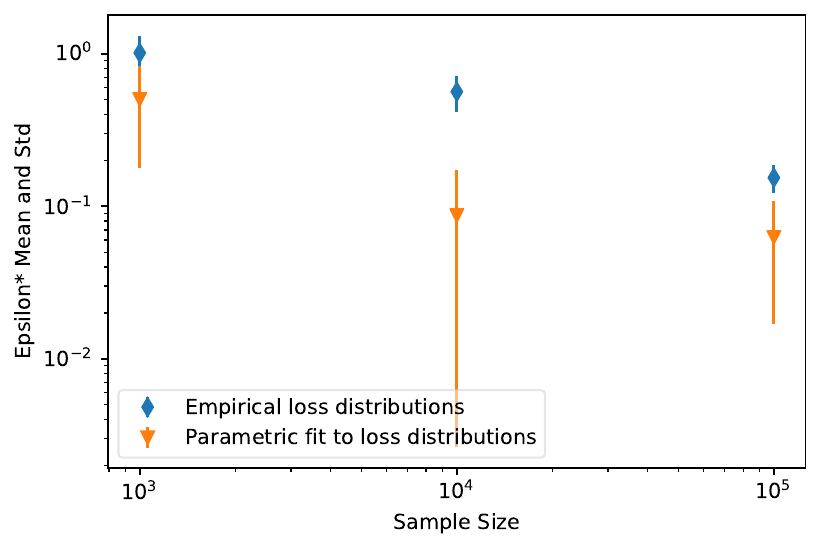}
  \caption{Epsilon* from empirical CDFs vs parametric fitted distributions to loss data when the true training and non-training distributions are identical}
  \label{fig:sim_d0}
\end{figure}

\subsubsection{Increasing the distance between training and non-training distributions}
We now explore in simulation what happens when the training and non-training loss distributions are farther and farther apart. We simulate this by increasing the parameter $d$, where the true training and non-training distributions are $\Gamma(k_1, \theta_1)$,  and  $\Gamma(k_1 + d, \theta_1)$ respectively,  for $d \in \{0,1,2,3\}$.  Since Figure \ref{fig:sim_d0} showed that the sample size can influence the accuracy of Epsilon*,  we also vary the training and non-training sample size $n$ in the set $\{10^3, 10^4, 10^5\}$.  

The mean and standard deviation of the Epsilon* values over 10 runs of the simulation are shown in Figure \ref{fig:vary_d_n}. This figure confirms the findings of Figure \ref{fig:sim_d0} when $d = 0$, meaning that the empirical CDF implementation overshoots the true value of 0 when $d = 0$ more than the parametric implementation,  but also shows that the empirical CDF implementation values undershoot the true Epsilon* values more than those of the parametric implementation as the distance $d$ between training and non-training loss distributions increases.  Just as in Figure \ref{fig:sim_d0}, this bias is exacerbated for smaller sample sizes but can be mitigated by increasing loss set sample sizes.  
\begin{figure}
  \centering
  \includegraphics[width=\linewidth]{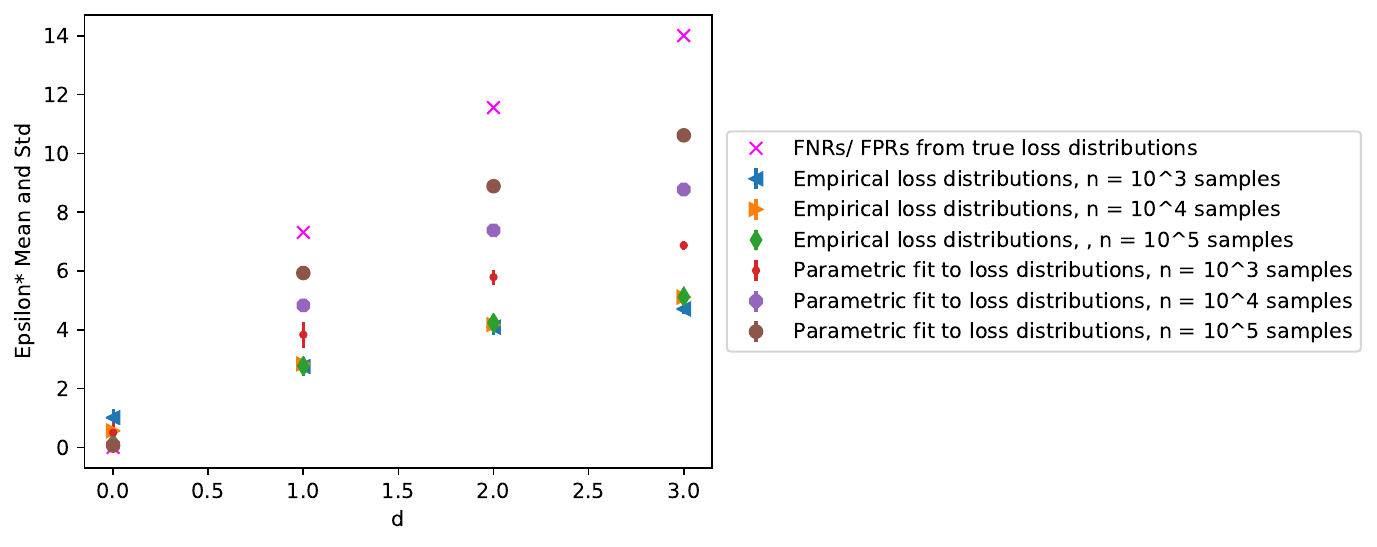}
  \caption{Epsilon* from empirical CDFs vs parametric fitted distributions to loss data when the true training and non-training distributions are $\Gamma(k_1, \theta_1)$,  and  $\Gamma(k_1 + d, \theta_1)$ respectively,  for $d \in \{0,1,2,3\}$,  $k_1 = 2$, $\theta_1$ = 5. }
  \label{fig:vary_d_n}
\end{figure}
We conclude from this simulation that fitting parametric distributions to loss data in order to estimate the FNRs and FPRs used in the Epsilon* computation leads to more accurate values than using empirical CDFs of the loss data sets.

\subsection{Goodness-of-fit for Parametric Distributions Fitted to Loss Data}
As described in Section \ref{subsec:simulation},  we fit Normal distributions to the transformed train and non-train losses in order to estimate the FNRs and FPRs required by the evaluation of Epsilon*.  In order to evaluate the quality of the parametric fitting,  we perform 2-sample Kolmogorov-Smirnov (KS) goodness-of-fit statistical tests which measure the maximum distance between the empirical CDFs of the two samples and tests the hypothesis that the underlying distributions are identical.  

We evaluate the goodness-of-fit test not only for a Normal distribution, but for more general Gaussian Mixture Models (GMMs) with up to 20 components.
To perform the KS test,  we:
\begin{enumerate}
\item Partition the  given loss set (be it training or non-training) into a fitting set and an evaluation set,  where the latter has size $n_{samples}$.
\item Fit a GMM with $n_{components}$ to the fitting set from Step 1.
\item Sample a set of $n_{samples}$ from the GMM fitted in Step 2.
\item Perform a 2-sample Kolmogorov-Smirnov (KS) test comparing the evaluation set to the set of samples generated by the GMM.
\end{enumerate}
We vary $n_{components}$  from  1 to 20, and $n_{samples}$ from 100 to 1,000. 

We are particularly interested in whether the p-value of the KS test goes above $\alpha$ =0.05 for any fitted GMM model, as that indicates that the null hypothesis (that the underlying distributions are identical) cannot be rejected at that confidence level.

\begin{figure*}[t!]
    \centering
    \begin{subfigure}[t]{0.5\textwidth}
        \centering
        \includegraphics[height = 2in]{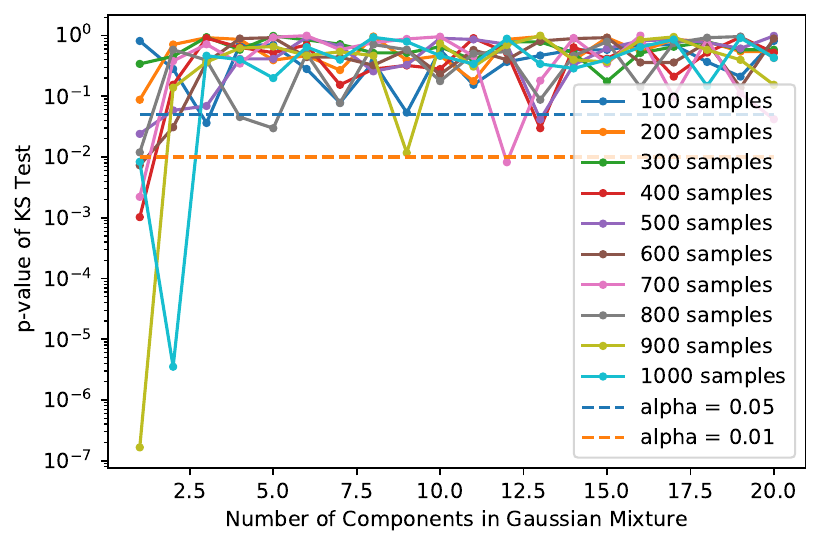}
        \caption{Model instance trained without DP-SGD on Adult data set}
    \end{subfigure}%
    ~ 
    \begin{subfigure}[t]{0.5\textwidth}
        \centering
        \includegraphics[height = 2in]{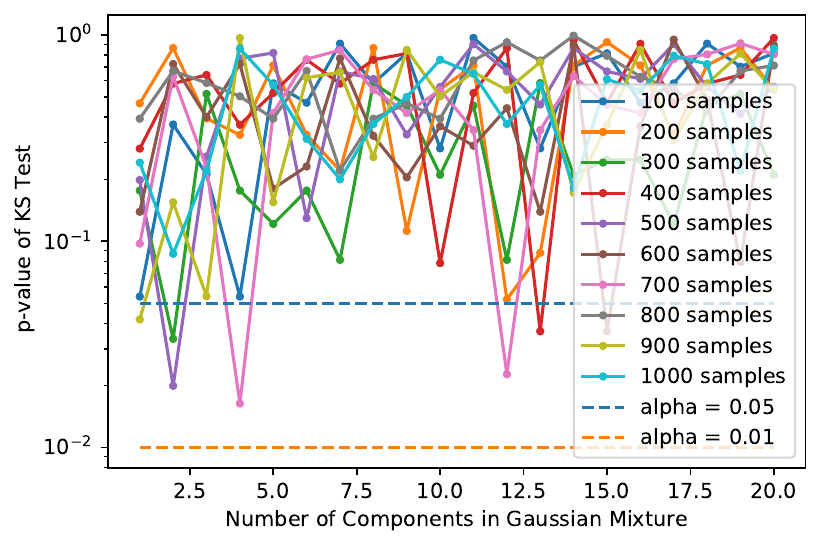}
        \caption{{\color{black}Model instance trained with DP-SGD on Adult data set}}
    \end{subfigure}
    \caption{p-values from KS tests of goodness-of-fit for a model instance of the Adult data set trained without (a) and with (b) DP-SGD}
    \label{fig:KS_adult_noDP_train}
\end{figure*}
Figure \ref{fig:KS_adult_noDP_train} shows the p-values of the KS test comparing GMMs of increasing number of components to a sample of the transformed training losses for instances of the Adult model trained with and without DP.  We find that most p-values go above 0.05, especially as the number of components in the GMM is increased, and when DP is used in training.  For most sample sizes ($n_{samples}$),  even only one component in the GMM results in p-values above 0.05,  especially when training with DP-SGD. Therefore,  in order to keep the complexity of the fitted distribution as low as possible, we decided to keep the one-component (Normal distribution) in our implementation. 

\end{document}